\documentclass[sigconf]{acmart}

\usepackage{booktabs} % For formal tables

\usepackage{amsmath}
\usepackage{amsfonts}
\usepackage{graphicx}
\usepackage{subcaption}
\usepackage{url}
\usepackage{amsbsy}
\usepackage{amssymb}
\usepackage{balance}

\begin{document}
\title[Combining Explicit and Implicit Feature Interactions for Recommender Systems]{xDeepFM: Combining Explicit and Implicit Feature Interactions for Recommender Systems}

\author[J. Lian]{Jianxun Lian}
\affiliation{%
  \institution{University of Science and Technology of China} 
}
\email{jianxun.lian@outlook.com}

\author[X. Zhou]{Xiaohuan Zhou}
\affiliation{%
  \institution{Beijing University of Posts and Telecommunications}
}
\email{maggione@bupt.edu.cn}

\author[F. Zhang]{Fuzheng Zhang}
\affiliation{%
  \institution{Microsoft Research}
  %\city{Haidian District, Beijing}
  %\country{China}
}
\email{fuzzhang@microsoft.com}

\author[Z. Chen]{Zhongxia Chen}
\affiliation{%
  \institution{University of Science and Technology of China}
  }
\email{czx87@mail.ustc.edu.cn}

\author[X. Xie]{Xing Xie}
\affiliation{%
  \institution{Microsoft Research}
  %\city{Haidian District, Beijing}
  %\country{China}
}
\email{xingx@microsoft.com}

\author[G. Sun]{Guangzhong Sun}
\affiliation{%
  \institution{University of Science and Technology of China} 
}
\email{gzsun@ustc.edu.cn}

\begin{abstract}
Combinatorial features are essential for the success of many commercial models. Manually crafting these features usually comes with high cost due to the variety, volume and velocity of raw data in web-scale systems. Factorization based models, which measure interactions in terms of vector product, can learn patterns of combinatorial features automatically and generalize to unseen features as well. With the great success of deep neural networks (DNNs) in various fields, recently researchers have proposed several DNN-based factorization model to learn both low- and high-order feature interactions. Despite the powerful ability of learning an arbitrary function from data, plain DNNs generate feature interactions implicitly and at the bit-wise level. In this paper, we propose a novel Compressed Interaction Network (CIN), which aims to generate feature interactions in an explicit fashion and at the vector-wise level. We show that the CIN share some functionalities with convolutional neural networks (CNNs) and recurrent neural networks (RNNs). We further combine a CIN and a classical DNN into one unified model, and named this new model eXtreme Deep Factorization Machine (xDeepFM). On one hand, the xDeepFM is able to learn certain bounded-degree feature interactions explicitly; on the other hand, it can learn arbitrary low- and high-order feature interactions implicitly. We conduct comprehensive experiments on three real-world datasets. Our results demonstrate that xDeepFM outperforms state-of-the-art models. We have released the source code of xDeepFM at \textsl{\url{https://github.com/Leavingseason/xDeepFM}}.
\end{abstract}





\copyrightyear{2018}
\acmYear{2018}
\setcopyright{acmcopyright}
\acmConference[KDD '18]{The 24th ACM SIGKDD International Conference on Knowledge Discovery \& Data Mining}{August 19--23, 2018}{London, United Kingdom}
\acmBooktitle{KDD '18: The 24th ACM SIGKDD International Conference on Knowledge Discovery \& Data Mining, August 19--23, 2018, London, United Kingdom}
\acmPrice{15.00}
\acmDOI{10.1145/3219819.3220023}
\acmISBN{978-1-4503-5552-0/18/08}

\maketitle

\section{introduction} \label{introduction}
Features play a central role in the success of many predictive systems. Because using raw features can rarely lead to optimal results, data scientists usually spend a lot of work on the transformation of raw features in order to generate best predictive systems \cite{he2014practical,lian2017restaurant} or to win data mining games \cite{liu2016repeat,Lian:2017:PLJ:3124791.3124794,lian2016cross}. One major type of feature transformation is the cross-product transformation over categorical features \cite{cheng2016wide}. These features are called \textsl{cross features} or \textsl{multi-way features}, they measure the interactions of multiple raw features. For instance, a 3-way feature {\fontfamily{qcr}\selectfont AND(user\_organization=msra, item\_category=deeplearning, time=monday)} has value 1 if the user works at Microsoft Research Asia and is shown a technical article about deep learning on a Monday.\\
\indent There are three major downsides for traditional cross feature engineering. First, obtaining high-quality features comes with a high cost. Because right features are usually task-specific, data scientists need spend a lot of time exploring the potential patterns from the product data before they become domain experts and extract meaningful cross features. Second, in large-scale predictive systems such as web-scale recommender systems, the huge number of raw features makes it infeasible to extract all cross features manually. Third, hand-crafted cross features do not generalize to unseen interactions in the training data. Therefore, learning to interact features without manual engineering is a meaningful task. \\
\indent Factorization Machines (FM) \cite{rendle2010factorization} embed each feature $i$ to a latent factor vector $\mathbf{v}_i = [v_{i1}, v_{i2}, ..., v_{iD}]$, and pairwise feature interactions are modeled as the inner product of latent vectors: $f^{(2)}(i,j)=\langle\mathbf{v}_i, \mathbf{v}_j\rangle x_ix_j$. In this paper we use the term \textsl{bit} to denote a element (such as $v_{i1}$) in latent vectors. The classical FM can be extended to arbitrary higher-order feature interactions \cite{blondel2016higher}, but one major downside is that, \cite{blondel2016higher} proposes to model all feature interactions, including both useful and useless combinations. As revealed in \cite{DBLP:conf/ijcai/XiaoY0ZWC17}, the interactions with useless features may introduce noises and degrade the performance. In recent years, deep neural networks (DNNs) have become successful in computer vision, speech recognition, and natural language processing with their great power of feature representation learning. It is promising to exploit DNNs to learn sophisticated and selective feature interactions. \cite{zhang2016deep} proposes a Factorisation-machine supported Neural Network (FNN) to learn high-order feature interactions. It uses the pre-trained factorization machines for field embedding before applying DNN. \cite{qu2016product} further proposes a Product-based Neural Network (PNN), which introduces a product layer between embedding layer and DNN layer, and does not rely on pre-trained FM. The major downside of FNN and PNN is that they focus more on high-order feature interactions while capture little low-order interactions. The Wide\&Deep \cite{cheng2016wide} and DeepFM \cite{guo2017deepfm} models overcome this problem by introducing hybrid architectures, which contain a shallow component and a deep component with the purpose of learning both memorization and generalization. Therefore they can jointly learn low-order and high-order feature interactions. \\
\indent All the abovementioned models leverage DNNs for learning high-order feature interactions. However, DNNs model high-order feature interactions in an implicit fashion. The final function learned by DNNs can be arbitrary, and there is no theoretical conclusion on what the maximum degree of feature interactions is. In addition, DNNs model feature interactions at the bit-wise level, which is different from the traditional FM framework which models feature interactions at the vector-wise level. Thus, in the field of recommender systems, whether DNNs are indeed the most effective model in representing high-order feature interactions remains an open question. In this paper, we propose a neural network-based model to learn feature interactions in an explicit, vector-wise fashion. Our approach is based on the Deep \& Cross Network (DCN) \cite{wang2017deep}, which aims to efficiently capture feature interactions of bounded degrees. However, we will argue in Section \ref{sec:explicit} that DCN will lead to a special format of interactions. We thus design a novel compressed interaction network (CIN) to replace the cross network in the DCN. CIN learns feature interactions explicitly, and the degree of interactions grows with the depth of the network. Following the spirit of the Wide\&Deep and DeepFM models, we combine the explicit high-order interaction module with implicit interaction module and traditional FM module, and name the joint model eXtreme Deep Factorization Machine (xDeepFM). The new model requires no manual feature engineering and release data scientists from tedious feature searching work. To summarize, we make the following contributions:
\begin{itemize}
    \item We propose a novel model, named eXtreme Deep Factorization Machine (xDeepFM), that jointly learns explicit and implicit high-order feature interactions effectively and requires no manual feature engineering.
    \item We design a compressed interaction network (CIN) in xDeepFM that learns high-order feature interactions explicitly. We show that the degree of feature interactions increases at each layer, and features interact at the vector-wise level rather than the bit-wise level.
    \item We conduct extensive experiments on three real-world dataset,  and the results demonstrate that our xDeepFM outperforms several state-of-the-art models significantly.
\end{itemize}
\indent The rest of this paper is organized as follows. Section \ref{preliminaries} provides some preliminary knowledge which is necessary for understanding deep learning-based recommender systems. Section \ref{mymodel} introduces our proposed CIN and xDeepFM model in detail. We will present experimental explorations on multiple datasets in Section \ref{experiments}. Related works are discussed in Section \ref{relatedwork}. Section \ref{conclusions} concludes this paper.
\section{Preliminaries}\label{preliminaries}
\subsection{Embedding Layer}\label{sec:emb}
In computer vision or natural language understanding, the input data are usually images or textual signals, which are known to be spatially and/or temporally correlated, so DNNs can be applied directly on the raw feature with dense structures. However, in web-scale recommender systems, the input features are sparse, of huge dimension, and present no clear spatial or temporal correlation. Therefore, \textsl{multi-field} categorical form is widely used by related works \cite{guo2017deepfm,shan2016deep,wang2017deep,qu2016product,zhang2016deep}. For example, one input instance \texttt{[user\_id=s02,gender=male,\\ organization=msra,interests=comedy\&rock]} is normally transformed into a high-dimensional sparse features via field-aware one-hot encoding:
\begin{equation*}
    [\underbrace{0,1,0,0,...,0}_{user id}]\ [\underbrace{1,0}_{gender}]\ [\underbrace{0,1,0,0,...,0}_{organization}]\ [\underbrace{0,1,0,1,...,0}_{interests}]
\end{equation*}
An embedding layer is applied upon the raw feature input to compress it to a low dimensional, dense real-value vector. If the field is univalent, the feature embedding is used as the field embedding. Take the above instance as an example, the embedding of feature \textsl{male} is taken as the embedding of field \textsl{gender}. If the field is multivalent, the sum of feature embedding is used as the field embedding. The embedding layer is illustrated in Figure \ref{fig:field_embedding}. The result of embedding layer is a wide concatenated vector:
\begin{equation*}
\mathbf{e}=[\mathbf{e}_1,\mathbf{e}_2,...,\mathbf{e}_m]    
\end{equation*}
where $m$ denotes the number of fields, and $\mathbf{e_i}\in \mathbb{R}^D$ denotes the embedding of one field. Although the feature lengths of instances can be various, their embeddings are of the same length $m\times D$, where $D$ is the dimension of field embedding.
\begin{figure}[htbp]
\centering
\includegraphics[width=0.4\textwidth]{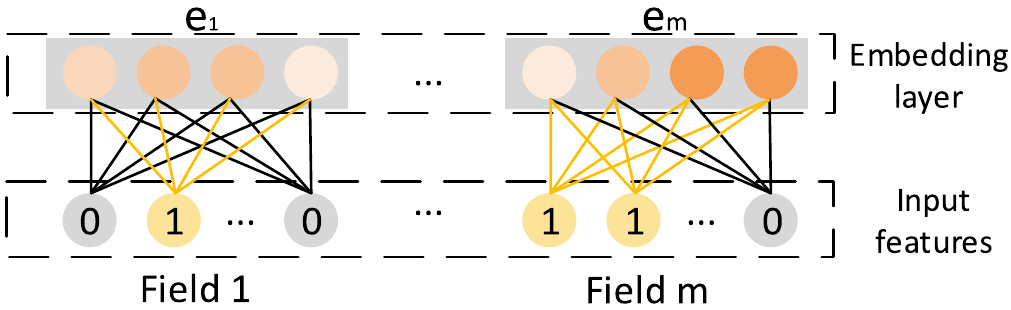}
\caption{The field embedding layer. The dimension of embedding in this example is 4.}
\label{fig:field_embedding}
\end{figure}
\subsection{Implicit High-order Interactions}\label{sec:implicit}
FNN \cite{zhang2016deep}, Deep Crossing \cite{shan2016deep}, and the deep part in Wide\&Deep \cite{cheng2016wide} exploit a feed-forward neural network on the field embedding vector $\mathbf{e}$ to learn high-order feature interactions. The forward process is :
\begin{equation}
    \mathbf{x}^1=\sigma (\mathbf{W}^{(1)}\mathbf{e}+\mathbf{b}^1)
\end{equation}
\begin{equation}
    \mathbf{x}^k=\sigma (\mathbf{W}^{(k)}\mathbf{x}^{(k-1)}+\mathbf{b}^k)
\end{equation}
where $k$ is the layer depth, $\sigma$ is an activation function, and $\mathbf{x}^k$ is the output of the $k$-th layer. The visual structure is very similar to what is shown in Figure \ref{fig:deepfm}, except that they do not include the \textsl{FM or Product layer}. This architecture models the interaction in a bit-wise fashion. That is to say, even the elements within the same field embedding vector will influence each other.  
\iffalse
\begin{figure}[htbp]
\centering
\includegraphics[width=0.4\textwidth]{materials/dnn_structure.pdf}
\caption{The architecture of multi-layer feed-forward neural network for high-order feature interactions.}
\label{fig:dnn_structure}
\end{figure}
\fi
\\
\indent PNN \cite{qu2016product} and DeepFM \cite{guo2017deepfm} modify the above architecture slightly. Besides applying DNNs on the embedding vector $\mathbf{e}$, they add a two-way interaction layer in the architecture. Therefore, both bit-wise and vector-wise interaction is included in their model. The major difference between PNN and DeepFM, is that PNN connects the outputs of product layer to the DNNs, whereas DeepFM connects the FM layer directly to the output unit (refer to Figure \ref{fig:deepfm}). 
\begin{figure}[htbp]
\centering
\includegraphics[width=0.48\textwidth]{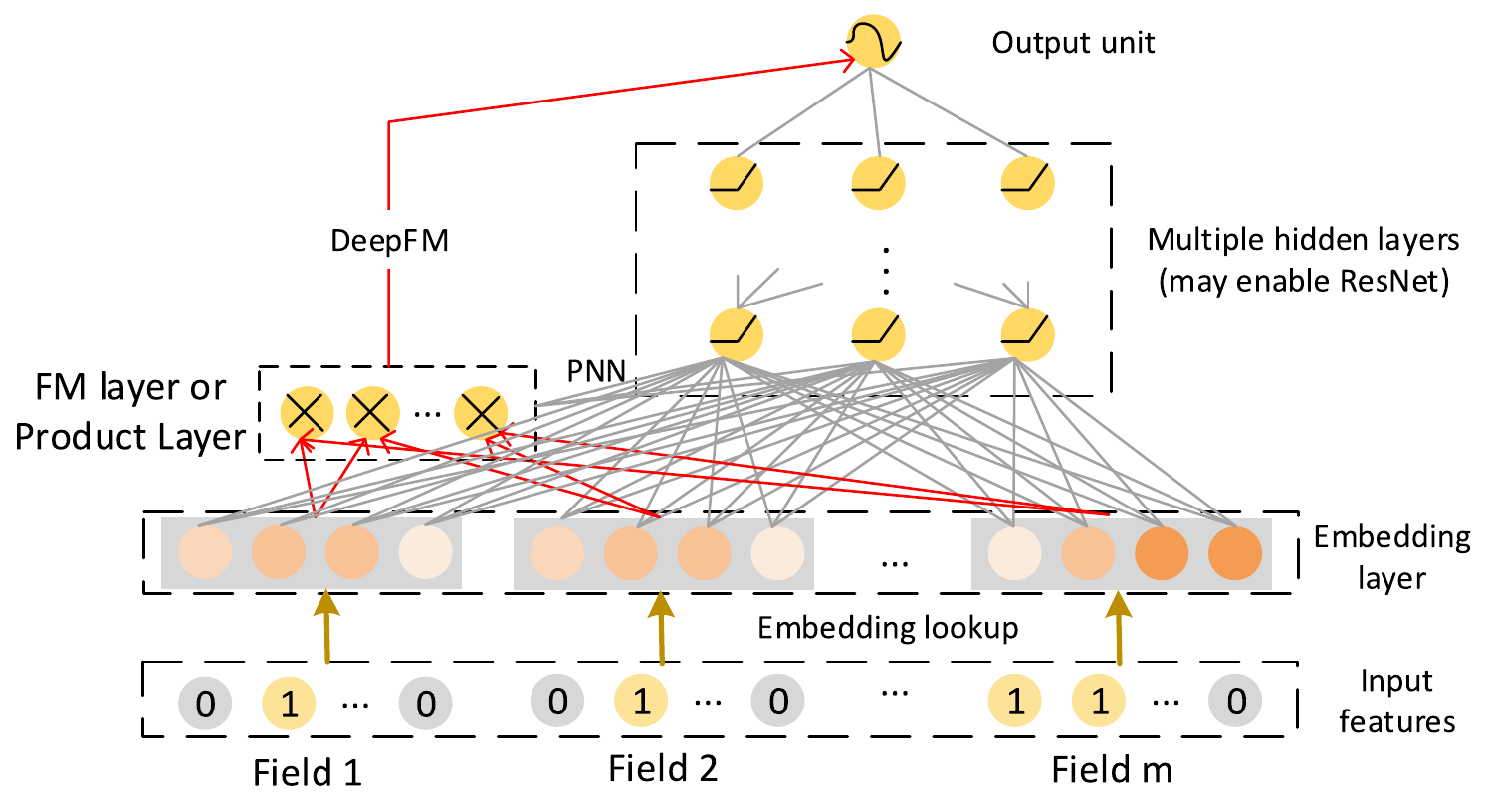}
\caption{The architecture of DeepFM (with linear part omitted) and PNN. We re-use the symbols in \cite{guo2017deepfm}, where red edges represent \textsl{weight-1 connections} (no parameters) and gray edges represent \textsl{normal connections} (network parameters).}
\label{fig:deepfm}
\end{figure}
\subsection{Explicit High-order Interactions}\label{sec:explicit}
\cite{wang2017deep} proposes the Cross Network (CrossNet) whose architecture is shown in Figure \ref{fig:crossnet}. It aims to explicitly model the high-order feature interactions. Unlike the classical fully-connected feed-forward network, the hidden layers are calculated by the following cross operation:
\begin{equation}
    \mathbf{x}_k = \mathbf{x}_0\mathbf{x}_{k-1}^T\mathbf{w}_k + \mathbf{b}_k + \mathbf{x}_{k-1}
\end{equation}
where $\mathbf{w}_k, \mathbf{b}_k, \mathbf{x}_k \in \mathbb{R}^{mD}$ are weights, bias and output of the $k$-th layer, respectively. We argue that the CrossNet learns a special type of high-order feature interactions, where each hidden layer in the CrossNet is a scalar multiple of $\mathbf{x}_{0}$.
\begin{figure}[htbp]
\centering
\includegraphics[width=0.31\textwidth]{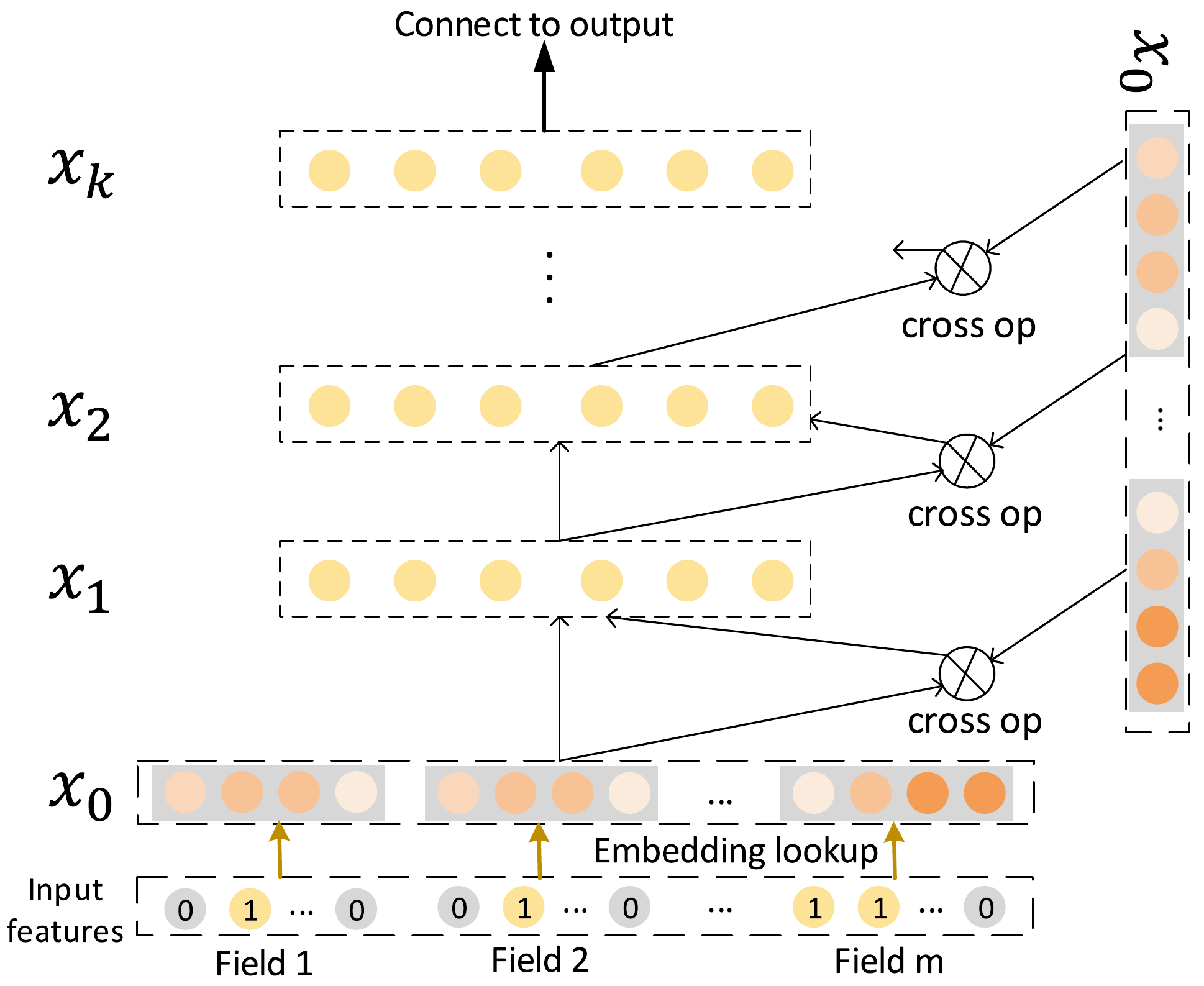}
\caption{The architecture of the Cross Network.}
\label{fig:crossnet}
\end{figure}
%\iffalse
\begin{theorem}
Consider a $k$-layer cross network with the (i+1)-th layer defined as $\mathbf{x}_{i+1} = \mathbf{x}_0\mathbf{x}_{i}^T\mathbf{w}_{i+1} + \mathbf{x}_{i}$. Then, the output of the cross network $\mathbf{x}_k$ is a scalar multiple of $\mathbf{x}_0$. 
\end{theorem} 
\begin{proof}
When $k$=1, according to the associative law and distributive law for matrix multiplication, we have:
\begin{equation}
\begin{split}
    \mathbf{x}_1 & = \mathbf{x}_0(\mathbf{x}_{0}^T\mathbf{w}_1) + \mathbf{x}_{0} \\
                & = \mathbf{x}_0(\mathbf{x}_{0}^T\mathbf{w}_1 + 1)  \\
                & =  \alpha^1 \mathbf{x}_0
\end{split}
\end{equation}
where the scalar $\alpha^1 = \mathbf{x}_{0}^T\mathbf{w}_1 + 1$ is actually a linear regression of $ \mathbf{x}_0$. Thus, $ \mathbf{x}_1$ is a scalar multiple of $\mathbf{x}_0$. Suppose the scalar multiple statement holds for $k$=$i$. For $k$=$i+1$, we have :
\begin{equation}
\begin{split}
    \mathbf{x}_{i+1} & = \mathbf{x}_0\mathbf{x}_{i}^T\mathbf{w}_{i+1} + \mathbf{x}_{i} \\
                & = \mathbf{x}_0((\alpha^{i}\mathbf{x}_{0})^T\mathbf{w}_{i+1}) + \alpha^{i}\mathbf{x}_{0} \\
                & =  \alpha^{i+1} \mathbf{x}_0
\end{split}
\end{equation}
where,  $\alpha^{i+1} = \alpha^i(\mathbf{x}_{0}^T\mathbf{w}_{i+1}+1)$ is a scalar. Thus $\mathbf{x}_{i+1}$ is still a scalar multiple of $\mathbf{x}_{0}$. By induction hypothesis, the output of cross network $\mathbf{x}_k$ is a scalar multiple of $\mathbf{x}_0$.
\end{proof}
%\fi
Note that the \textsl{scalar multiple} does not mean $\mathbf{x}_k$ is linear with $\mathbf{x}_0$. The coefficient $\alpha^{i+1}$ is sensitive with $\mathbf{x}_0$. The CrossNet can learn feature interactions very efficiently (the complexity is negligible compared with a DNN model), however the downsides are: (1) the output of CrossNet is limited in a special form, with each hidden layer is a scalar multiple of $\mathbf{x}_{0}$; (2) interactions come in a bit-wise fashion.

\section{Our proposed model}\label{mymodel}
\begin{figure*}[htbp]
\centering
\begin{subfigure}{.33\textwidth}
  \centering
  \includegraphics[width=0.85\textwidth]{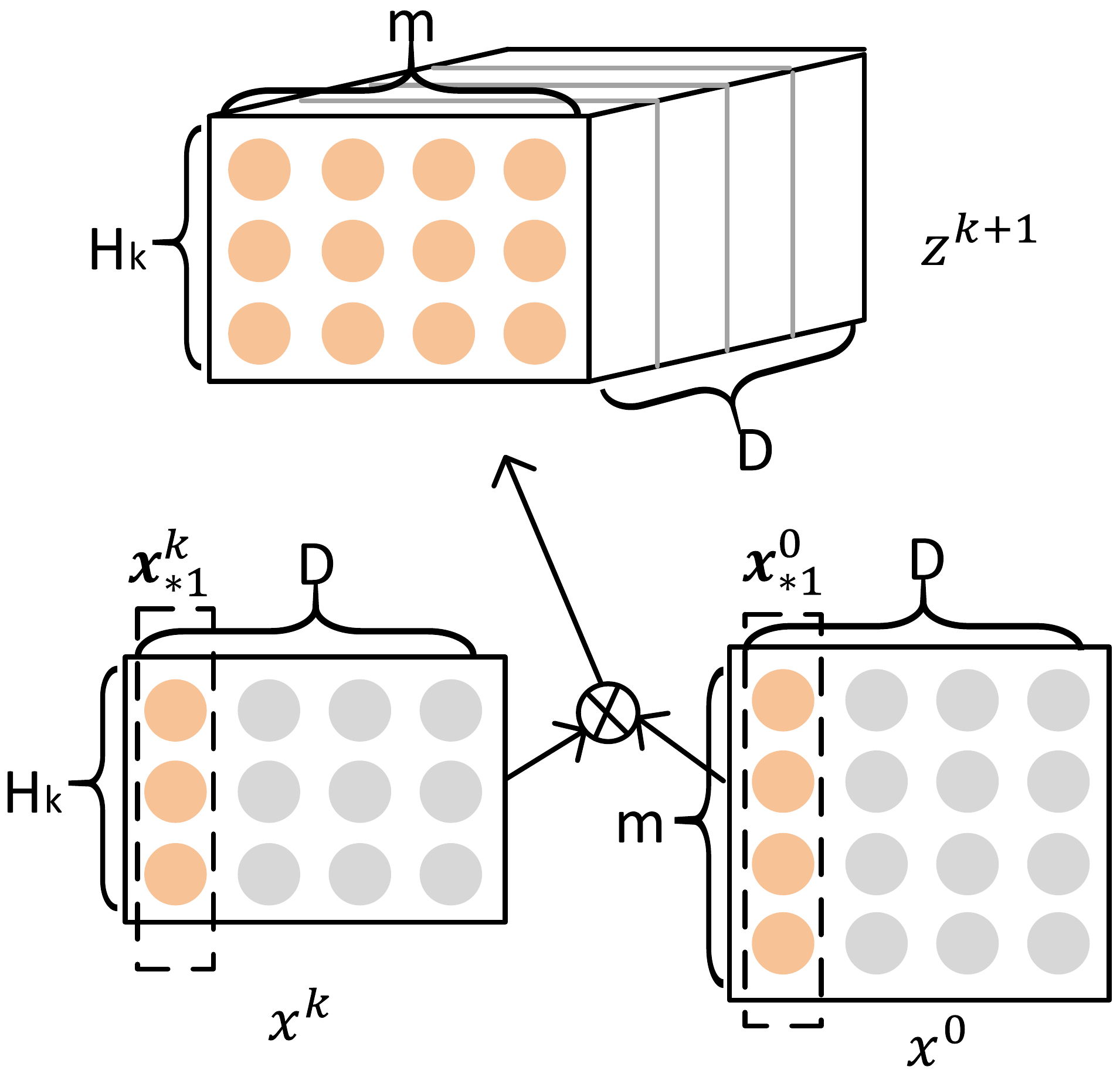}
  \caption{Outer products along each dimension for feature interactions. The tensor $\mathbf{Z}^{k+1}$ is an intermediate result for further learning.}
  \label{fig:CIN1}
\end{subfigure} \hfill  
\begin{subfigure}{.32\textwidth}
  \centering
  \includegraphics[width=0.85\textwidth]{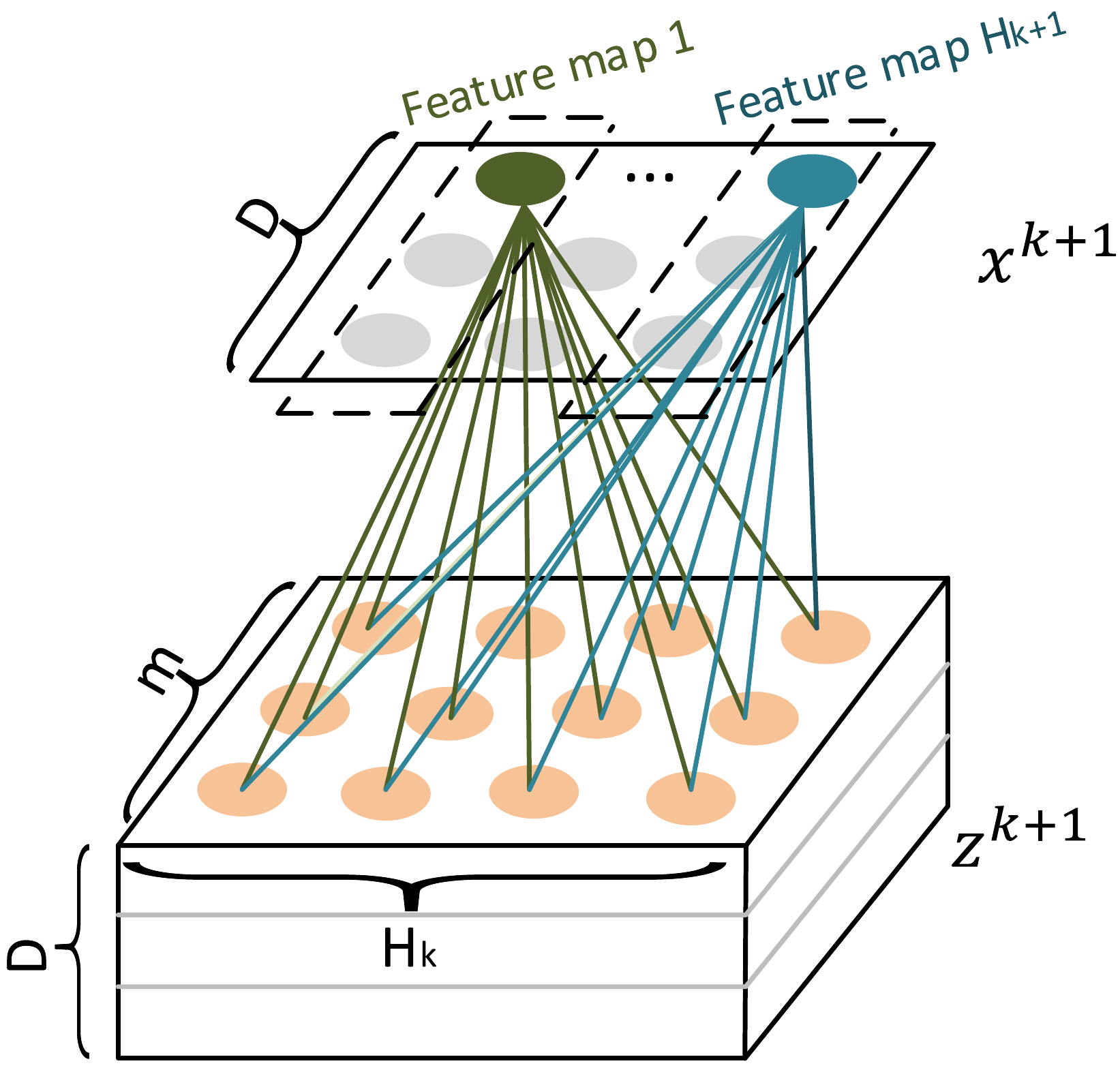}
  \caption{The $k$-th layer of CIN. It compresses the intermediate tensor $\mathbf{Z}^{k+1}$ to $H_{k+1}$ embedding vectors (aslo known as \textsl{feature maps}).}
  \label{fig:CIN2}
\end{subfigure}   
\begin{subfigure}{.32\textwidth}
  \centering
  \includegraphics[width=0.85\textwidth]{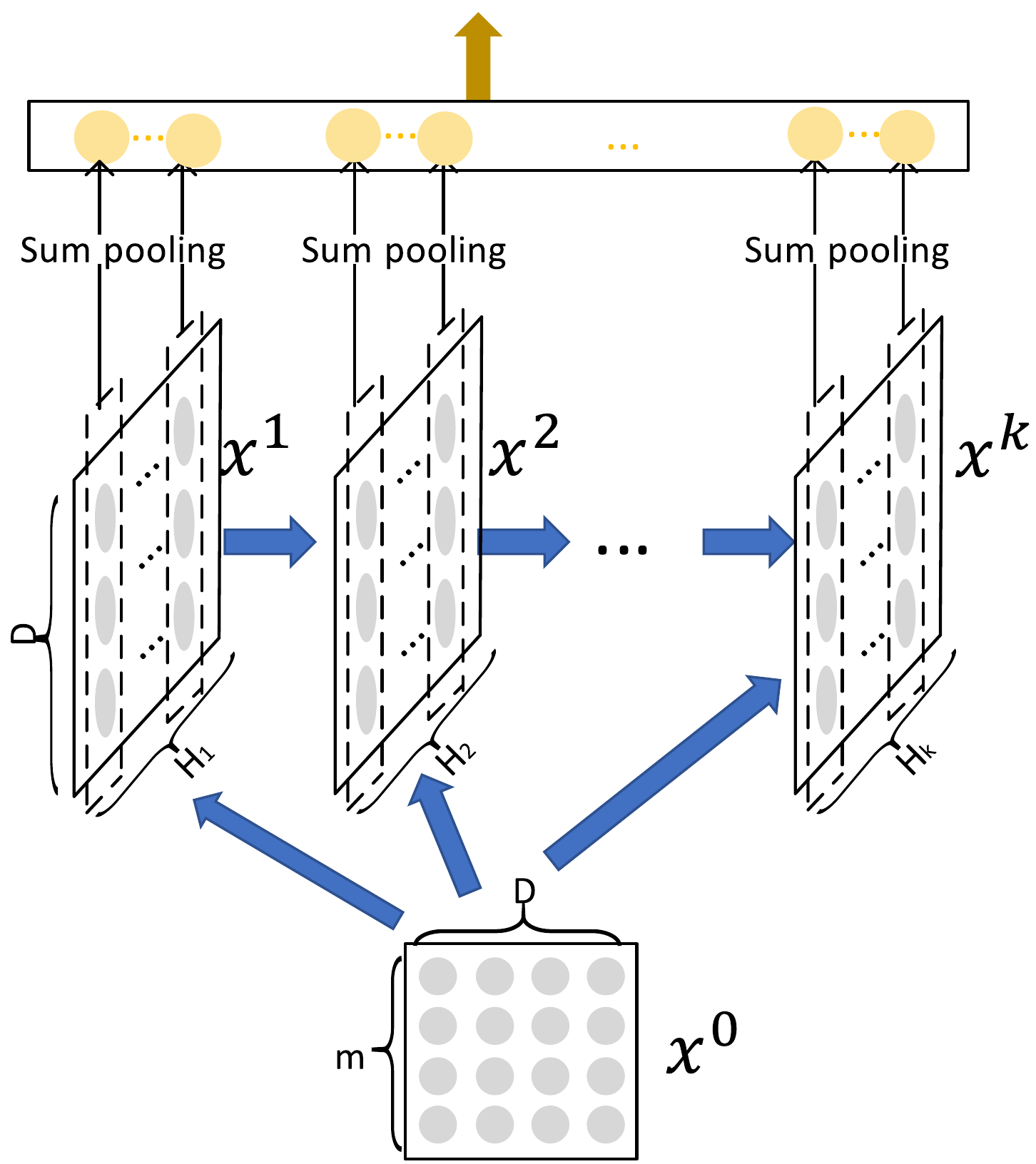}
  \caption{An overview of the CIN architecture.}
  \label{fig:CIN3}
\end{subfigure} 
\caption{Components and architecture of the Compressed Interaction Network (CIN).}
\label{fig:CIN}
\end{figure*}
%In this section, we first introduce a novel Compressed Interaction Network (CIN), which is more effective in explicit high-order feature interactions than the CrossNet. Then we introduce the eXtreme Deep Factorization Machine (xDeepFM) which jointly train an explicit and an implicit high-order interaction network.
\subsection{Compressed Interaction Network}
We design a new cross network, named Compressed Interaction Network (CIN), with the following considerations: (1) interactions are applied at vector-wise level, not at bit-wise level; (2) high-order feature interactions is measured explicitly; (3) the complexity of network will not grow exponentially with the degree of interactions.\\
\indent Since an embedding vector is regarded as a unit for vector-wise interactions, hereafter we formulate the output of field embedding as a matrix $\mathbf{X}^0 \in \mathbb{R}^{m\times D}$, where the $i$-th row in $\mathbf{X}^0$ is the embedding vector of the $i$-th field: $\mathbf{X}^0_{i,*} = \mathbf{e}_i$, and $D$ is the dimension of the field embedding. The output of the $k$-th layer in CIN is also a matrix $\mathbf{X}^k \in \mathbb{R}^{H_k\times D}$, where $H_k$ denotes the number of (embedding) feature vectors in the $k$-th layer and we let $H_0=m$. For each layer, $\mathbf{X}^k$ are calculated via:
\begin{equation}
\label{eq:xk}
    \mathbf{X}^k_{h,*} = \sum_{i=1}^{H_{k-1}}\sum_{j=1}^{m}\mathbf{W}^{k,h}_{ij}(\mathbf{X}^{k-1}_{i,*}\circ \mathbf{X}^0_{j,*})  
\end{equation}
where $1\leq h \leq H_k$, $\mathbf{W}^{k,h}\in \mathbb{R}^{H_{k-1}\times m}$ is the parameter matrix for the $h$-th feature vector, and $\circ$ denotes the Hadamard product, for example, $\langle a_1,a_2,a_3 \rangle \circ \langle b_1,b_2,b_3\rangle = \langle a_1b_1,a_2b_2,a_3b_3\rangle$. Note that $\mathbf{X}^k$ is derived via the interactions between $\mathbf{X}^{k-1}$ and $\mathbf{X}^0$, thus feature interactions are measured explicitly and the degree of interactions increases with the layer depth. The structure of CIN is very similar to the Recurrent Neural Network (RNN), where the outputs of the next hidden layer are dependent on the last hidden layer and an additional input. We hold the structure of embedding vectors at all layers, thus the interactions are applied at the vector-wise level.\\
\indent It is interesting to point out that Equation \ref{eq:xk} has strong connections with the well-known Convolutional Neural Networks (CNNs) in computer vision. As shown in Figure \ref{fig:CIN1}, we introduce an intermediate tensor $\mathbf{Z}^{k+1}$, which is the outer products (along each embedding dimension) of hidden layer $\mathbf{X}^k$ and original feature matrix $\mathbf{X}^0$. Then $\mathbf{Z}^{k+1}$ can be regarded as a special type of image and  $\mathbf{W}^{k,h}$ is a filter. We slide the filter across $\mathbf{Z}^{k+1}$ along the embedding dimension (D) as shown in Figure \ref{fig:CIN2}, and get an hidden vector $\mathbf{X}^{k+1}_{i,*}$, which is usually called a \textsl{feature map} in computer vision. Therefore, $\mathbf{X}^k$ is a collection of $H_k$ different \textsl{feature maps}. The term ``\textsl{compressed}" in the name of CIN indicates that the $k$-th hidden layer compress the potential space of $H_{k-1}\times m$ vectors down to $H_k$ vectors.\\
\indent Figure \ref{fig:CIN3} provides an overview of the architecture of CIN. Let T denotes the depth of the network. Every hidden layer $\mathbf{X}^k, k\in [1,T]$ has a connection with output units. We first apply sum pooling on each feature map of the hidden layer:
\begin{equation}
    {p}^k_i = \sum_{j=1}^D \mathbf{X}^k_{i,j}
\end{equation}
for $i\in [1,H_k]$. Thus, we have a pooling vector $\mathbf{p}^k = [p^k_1, p^k_2, ..., p^k_{H_k}]$ with length $H_k$ for the $k$-th hidden layer. All pooling vectors from hidden layers are concatenated before connected to output units: $\mathbf{p}^+ = [\mathbf{p}^1,\mathbf{p}^2,...,\mathbf{p}^T] \in \mathbb{R}^{\sum_{i=1}^T H_i}$. If we use CIN directly for binary classification, the output unit is a sigmoid node on $\mathbf{p}^+$:
\begin{equation}
    y=\frac{1}{1+exp(\mathbf{p^+}^T\mathbf{w}^{o})}
\end{equation}
where $\mathbf{w}^{o}$ are the regression parameters.
\subsection{CIN Analysis}
We analyze the proposed CIN to study the model complexity and the potential effectiveness.
\subsubsection{Space Complexity}
The $h$-th feature map at the $k$-th layer contains $H_{k-1} \times m$ parameters, which is exactly the size of $\mathbf{W}^{k,h}$. Thus, there are $H_{k} \times H_{k-1} \times m$ parameters at the $k$-th layer. Considering the last regression layer for the output unit, which has $\sum_{k=1}^{T}H_{k}$ parameters, the total number of parameters for CIN is $\sum_{k=1}^{T}H_{k} \times (1 + H_{k-1} \times m)$. Note that CIN is independent of the embedding dimension $D$. In contrast, a plain $T$-layers DNN contains $m \times D \times H_1 + H_T + \sum_{k=2}^{T}H_{k} \times H_{k-1}$ parameters, and the number of parameters will increase with the embedding dimension $D$. \\
\indent Usually $m$ and $H_k$ will not be very large, so the scale of $\mathbf{W}^{k,h}$ is acceptable. When necessary, we can exploit a $L$-order decomposition and replace $\mathbf{W}^{k,h}$ with two smaller matrices $\mathbf{U}^{k,h} \in \mathbb{R}^{H_{k-1} \times L}$ and $\mathbf{V}^{k,h} \in \mathbb{R}^{m\times L}$:
\begin{equation}
    \mathbf{W}^{k,h} = \mathbf{U}^{k,h} (\mathbf{V}^{k,h})^T
\end{equation}
where $L\ll H$ and $L\ll m$. Hereafter we assume that each hidden layer has the same number (which is $H$) of feature maps for simplicity. Through the $L$-order decomposition, the space complexity of CIN is reduced from $O(mTH^2)$ to $O(mTHL+TH^2L)$. In contrast, the space complexity of the plain DNN is $O(mDH+TH^2)$, which is sensitive to the dimension (D) of field embedding. 
\subsubsection{Time Complexity}\label{sec:time_complexity}
The cost of computing tensor $\mathbf{Z}^{k+1}$ (as shown in Figure \ref{fig:CIN1}) is $O(mHD)$ time. Because we have $H$ feature maps in one hidden layer, computing a $T$-layers CIN takes $O(mH^2DT)$ time. A $T$-layers plain DNN, by contrast, takes $O(mHD + H^2T)$ time. Therefore, the major downside of CIN lies in the time complexity. 
\subsubsection{Polynomial Approximation}
Next we examine the high-order interaction properties of CIN. For simplicity, we assume that numbers of feature maps at hidden layers are all equal to the number of fields $m$. Let $[m]$ denote the set of positive integers that are less than or equal to $m$. The $h$-th feature map at the first layer, denoted as  $\mathbf{x}^{1}_h\ \in \mathbb{R}^{D}$, is calculated via:
\begin{equation}
    \mathbf{x}^{1}_h = \sum_{\substack{i\in [m]\\ j\in [m]}} \mathbf{W}^{1,h}_{i,j}(\mathbf{x}^0_i \circ \mathbf{x}^0_j)
\end{equation}
Therefore, each feature map at the first layer models pair-wise interactions with $O(m^2)$ coefficients. Similarly, the $h$-th feature map at the second layer is:
\begin{equation}
\label{eq:level2}
\begin{split}
    \mathbf{x}^{2}_h & = \sum_{\substack{i\in [m]\\ j\in [m]}} \mathbf{W}^{2,h}_{i,j}(\mathbf{x}^1_i \circ \mathbf{x}^0_j) \\
        & = \sum_{\substack{i\in [m]\\ j\in [m]}} \sum_{\substack{l\in [m]\\ k\in [m]}} \mathbf{W}^{2,h}_{i,j}\mathbf{W}^{1,i}_{l,k} (\mathbf{x}^0_j \circ \mathbf{x}^0_k \circ \mathbf{x}^0_l)
\end{split}
\end{equation}
Note that all calculations related to the subscript $l$ and $k$ is already finished at the previous hidden layer. We expand the factors in Equation \ref{eq:level2} just for clarity. We can observe that each feature map at the second layer models 3-way interactions with $O(m^2)$ new parameters. \\
\indent A classical $k$-order polynomial has $O(m^k)$ coefficients. We show that CIN approximate this class of polynomial with only $O(km^3)$ parameters in terms of a chain of feature maps. By induction hypothesis, we can prove that the $h$-th feature map at the $k$-th layer is:
\begin{equation}
\label{eq:levelk}
\begin{split}
    \mathbf{x}^{k}_h & = \sum_{\substack{i\in [m]\\ j\in [m]}} \mathbf{W}^{k,h}_{i,j}(\mathbf{x}^{k-1}_i \circ \mathbf{x}^0_j) \\
        & = \sum_{\substack{i\in [m]\\ j\in [m]}} ... \sum_{\substack{r\in [m]\\ t\in [m]}} \sum_{\substack{l\in [m]\\ s\in [m]}} \mathbf{W}^{k,h}_{i,j}...\mathbf{W}^{1,r}_{l,s} (\underbrace{\mathbf{x}^0_j \circ ... \circ \mathbf{x}^0_s \circ \mathbf{x}^0_l}_{k\ vectors})
\end{split}
\end{equation}
For better illustration, here we borrow the notations from \cite{wang2017deep}. Let $\boldsymbol{\alpha}=[\alpha_1,...,\alpha_m] \in \mathbb{N}^d$ denote a multi-index, and $|\boldsymbol{\alpha}|=\sum_{i=1}^m \alpha_i$. We omit the original superscript from $\mathbf{x}^0_i$, and use $\mathbf{x}_i$ to denote it since we only we the feature maps from the $0$-th layer (which is exactly the field embeddings) for the final expanded expression (refer to Eq. \ref{eq:levelk}). Now a superscript is used to denote the vector operation, such as $\mathbf{x}^3_i=\mathbf{x}_i \circ \mathbf{x}_i \circ \mathbf{x}_i$. Let $VP_k(\mathbf{X})$ denote a multi-vector polynomial of degree $k$:
\begin{equation}  
VP_k(\mathbf{X})= \left\{ \left. \sum_{\boldsymbol{\alpha}} w_{\boldsymbol{\alpha}} \mathbf{x}_1^{\alpha_1} \circ \mathbf{x}_2^{\alpha_2} \circ ... \circ \mathbf{x}_m^{\alpha_m}  \right| 2 \leqslant |\boldsymbol{\alpha}| \leqslant k \right\}
\end{equation}
Each vector polylnomial in this class has $O(m^k)$ coefficients. Then, our CIN approaches the coefficient $w_{\boldsymbol{\alpha}}$ with:
\begin{equation} 
 \hat w_{\boldsymbol{\alpha}} = \sum_{i=1}^{m} \sum_{j=1}^{m} \sum_{B\in P_{\boldsymbol{\alpha}}} \prod_{t=2}^{|\boldsymbol{\alpha}|} \mathbf{W}^{t,j}_{i,B_t}
\end{equation}
where, $B=[B_1, B_2, ..., B_{|\boldsymbol{\alpha}|}]$ is a multi-index, and $P_{\boldsymbol{\alpha}}$ is the set of all the permutations of the indices $ (\underbrace{1,...1}_{\alpha_1\ times},...,\underbrace{m,...,m}_{\alpha_m\ times})$.

\subsection{Combination with Implicit Networks}
As discussed in Section \ref{sec:implicit}, plain DNNs learn implicit high-order feature interactions. Since CIN and plain DNNs can complement each other, an intuitive way to make the model stronger is to combine these two structures. The resulting model is very similar to the Wide\&Deep or DeepFM model. The architecture is shown in Figure \ref{fig:xDeepFM}. We name the new model eXtreme Deep Factorization Machine (xDeepFM), considering that on one hand, it includes both low-order and high-order feature interactions; on the other hand, it includes both implicit feature interactions and explicit feature interactions. Its resulting output unit becomes:
\begin{equation}
    \hat y=\sigma(\mathbf{w}_{linear}^T\mathbf{a}+\mathbf{w}^T_{dnn}\mathbf{x}^k_{dnn}+\mathbf{w}^T_{cin}\mathbf{p}^+ + b)
\end{equation}
where $\sigma$ is the sigmoid function, $\mathbf{a}$ is the raw features. $\mathbf{x}^k_{dnn}, \mathbf{p}^+$ are the outputs of the plain DNN and CIN, respectively. $\mathbf{w}_*$ and $b$ are learnable parameters. For binary classifications, the loss function is the log loss:
\begin{equation}
    \mathcal{L}=-\frac{1}{N}\sum_{i=1}^{N}y_i log \hat y_i + (1-y_i)log(1-\hat y_i)
\end{equation}
where $N$ is the total number of training instances. The optimization process is to minimize the following objective function:
\begin{equation}
    \mathcal{J}=\mathcal{L} + \lambda_* ||\Theta||
\end{equation}
where $\lambda_*$ denotes the regularization term and $\Theta$ denotes the set of parameters, including these in linear part, CIN part, and DNN part.
\begin{figure}[htbp]
\centering
\includegraphics[width=0.45\textwidth]{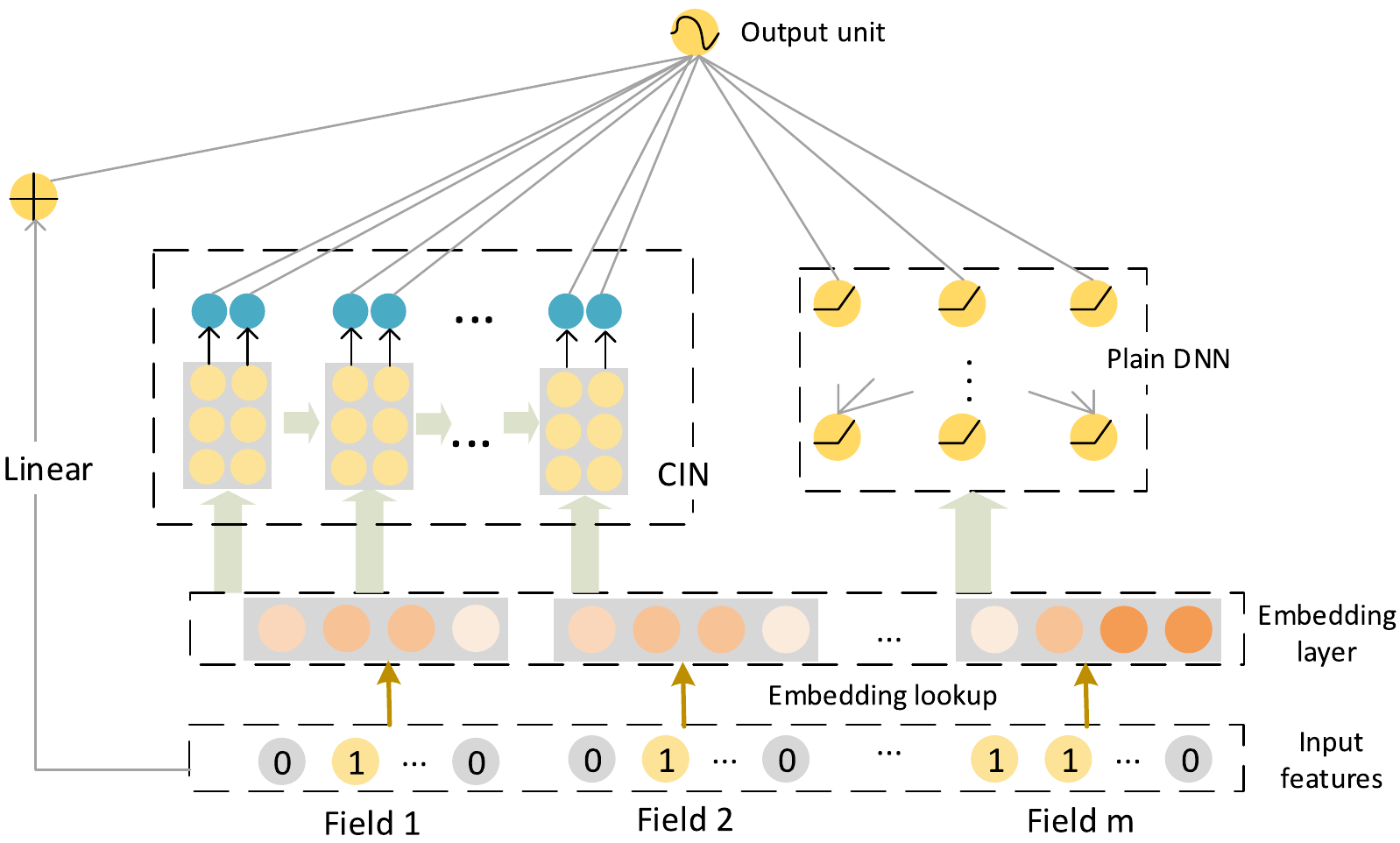}
\caption{The architecture of xDeepFM.}
\label{fig:xDeepFM}
\end{figure}
\subsubsection{Relationship with FM and DeepFM}
Suppose all fields are univalent. It's not hard to observe from Figure \ref{fig:xDeepFM} that, when the depth and feature maps of the CIN part are both set to 1, xDeepFM is a generalization of DeepFM by learning the linear regression weights for the FM layer (note that in DeepFM, units of FM layer are directly linked to the output unit without any coefficients). When we further remove the DNN part, and at the same time use a constant \textsl{sum filter} (which simply takes the sum of inputs without any parameter learning) for the feature map, then xDeepFM is downgraded to the traditional FM model.
\section{Experiments}\label{experiments}
In this section, we conduct extensive experiments to answer the following questions:
\begin{itemize}
    \item \textbf{(Q1)} How does our proposed CIN perform in high-order feature interactions learning?
    \item \textbf{(Q2)} Is it necessary to combine explicit and implicit high-order feature interactions for recommender systems?
    \item \textbf{(Q3)} How does the settings of networks influence the performance of xDeepFM?
\end{itemize}
We will answer these questions after presenting some fundamental experimental settings.
\subsection{Experiment Setup}
\subsubsection{Datasets.} We evaluate our proposed models on the following three datasets:\\
\indent \textbf{1. Criteo Dataset}. It is a famous industry benchmarking dataset for developing models predicting ad click-through rate, and is publicly accessible\footnote{http://labs.criteo.com/2014/02/kaggle-display-advertising-challenge-dataset/}. Given a user and the page he is visiting, the goal is to predict the probability that he will clik on a given ad. \\
\indent \textbf{2. Dianping Dataset}. \textsl{Dianping.com} is the largest consumer review site in China. It provides diverse functions such as reviews, check-ins, and shops' meta information (including geographical messages and shop attributes). We collect 6 months' users check-in activities for restaurant recommendation experiments. Given a user's profile, a restaurant's attributes and the user's last three visited POIs (point of interest), we want to predict the probability that he will visit the restaurant. For each restaurant in a user's check-in instance, we sample four restaurants which are within 3 kilometers as negative instances by POI popularity.\\
\indent \textbf{3. Bing News Dataset}. Bing News\footnote{https://www.bing.com/news} is part of Microsoft's Bing search engine. In order to evaluate the performance of our model in a real commercial dataset, we collect five consecutive days' impression logs on news reading service. We use the first three days' data for training and validation, and the next two days for testing.  \\
\indent For the Criteo dataset and the Dianping dataset, we randomly split instances by 8:1:1 for training , validation and test. The characteristics of the three datasets are summarized in Table \ref{tab:datasets}.
  \begin{table}[ht]
     \centering
     \caption{Statistics of the evaluation datasets. M indicates million and K indicates thousand. }
     \begin{tabular}{|c|c|c|c|} \hline
        Datasest  &  \#instances & \#fields &\#features (sparse)  \\ \hline    
        Criteo    &  45M & 39 & 2.3M \\ \hline           
        Dianping &  1.2M & 18 & 230K \\ \hline   
        Bing News &  5M & 45 & 17K \\  \hline
     \end{tabular}
     \label{tab:datasets} 
 \end{table}
\subsubsection{Evaluation Metrics.} We use two metrics for model evaluation: \textbf{AUC} (Area Under the ROC curve) and \textbf{Logloss} (cross entropy). These two metrics evaluate the performance from two different angels: AUC measures the probability that a positive instance will be ranked higher than a randomly chosen negative one. It only takes into account the order of predicted instances and is insensitive to class imbalance problem. Logloss, in contrast, measures the distance between the predicted score and the true label for each instance. Sometimes we rely more on Logloss because we need to use the predicted probability to estimate the benefit of a ranking strategy (which is usually adjusted as CTR $\times$ bid).
\subsubsection{Baselines.} We compare our xDeepFM with LR(logistic regression), FM, DNN (plain deep neural network), PNN (choose the better one from iPNN and oPNN) \cite{qu2016product}, Wide \& Deep \cite{cheng2016wide}, DCN (Deep \& Cross Network) \cite{wang2017deep} and DeepFM \cite{guo2017deepfm}. As introduced and discussed in Section \ref{preliminaries}, these models are highly related to our xDeepFM and some of them are state-of-the-art models for recommender systems. Note that the focus of this paper is to learn feature interactions automatically, so we do not include any hand-crafted cross features.  
\subsubsection{Reproducibility} We implement our method using Tensorflow\footnote{https://www.tensorflow.org/}. Hyper-parameters of each model are tuned by grid-searching on the validation set, and the best settings for each model will be shown in corresponding sections. Learning rate is set to \textsl{0.001}. For optimization method, we use the Adam \cite{kingma2014adam} with a mini-batch size of 4096. We use a L2 regularization with $\lambda=0.0001$ for DNN, DCN, Wide\&Deep, DeepFM and xDeepFM, and use dropout 0.5 for PNN. The default setting for number of neurons per layer is: (1) 400 for DNN layers; (2) 200 for CIN layers on Criteo dataset, and 100 for CIN layers on Dianping and Bing News datasets. Since we focus on neural networks structures in this paper, we make the dimension of field embedding for all models be a fixed value of 10. We conduct experiments of different settings in parallel with 5 \textsl{Tesla K80 GPUs}.
The source code is available at \textsl{\url{https://github.com/Leavingseason/xDeepFM}}. 

  \begin{table}[ht]
     \centering
     \caption{Performance of individual models on the Criteo, Dianping, and Bing News datasets. Column \textsl{Depth} indicates the best network depth for each model.  }
     \begin{tabular}{c|cc|c} \hline\hline  
        Model name &  AUC & Logloss & Depth \\ \hline  
        \multicolumn{4}{c}{Criteo}\\ \hline  
        FM    & 0.7900   & 0.4592  &  -  \\ \hline           
        DNN & 0.7993  &  0.4491 &  2 \\ \hline   
        CrossNet &  0.7961 & 0.4508 & 3 \\  \hline
        CIN & \textbf{0.8012}  &  0.4493 &  3 \\ \hline   
         \multicolumn{4}{c}{Dianping}\\ \hline  
         FM    & 0.8165  & 0.3558  &  -  \\ \hline           
        DNN & 0.8318  &  0.3382 &  3 \\ \hline   
        CrossNet &  0.8283 & 0.3404 & 2 \\  \hline
        CIN & \textbf{0.8576}  &  \textbf{0.3225} &  2 \\ \hline    
          \multicolumn{4}{c}{Bing News}\\ \hline  
         FM    & 0.8223 & 0.2779  &  -  \\ \hline           
        DNN & 0.8366  &  0.273 &  2 \\ \hline   
        CrossNet &  0.8304 & 0.2765 & 6 \\  \hline
        CIN & \textbf{0.8377}  &  \textbf{0.2662} &  5 \\ \hline \hline  
     \end{tabular}
     \label{tab:cin} 
 \end{table}
   \begin{table*}[th]
    \centering
    \caption{Overall performance of different models on Criteo, Dianping and Bing News datasets. The column \textsl{Depth} presents the best setting for network depth with a format of (cross layers, DNN layers).}
     {\setlength{\tabcolsep}{1em}{\renewcommand{\arraystretch}{1.1}
     \begin{tabular}{c|cc|c||cc|c||cc|c} \hline\hline  
         & \multicolumn{3}{c||}{Criteo} & \multicolumn{3}{c||}{Dianping} & \multicolumn{3}{c}{Bing News}  \\ \hline  
        Model name &  AUC & Logloss & Depth&  AUC & Logloss & Depth &  AUC & Logloss & Depth  \\ \hline  
        LR  & 0.7577 & 0.4854   &    -,-  & 0.8018 & 0.3608  &    -,-  & 0.7988 & 0.2950  &    -,-   \\ \hline       
        FM    & 0.7900  & 0.4592  &  -,-  & 0.8165  & 0.3558  &  -,- & 0.8223 & 0.2779 &  -,-  \\ \hline    
        DNN & 0.7993  &  0.4491 &  -,2  & 0.8318  &  0.3382 & -,3 & 0.8366  &  0.2730 & -,2 \\ \hline   
        DCN &  0.8026 & 0.4467 & 2,2 &  0.8391 & 0.3379 & 4,3&  0.8379 & 0.2677 & 2,2  \\  \hline
        Wide\&Deep & 0.8000 & 0.4490 & -,3& 0.8361 & 0.3364 & -,2 & 0.8377 & 0.2668 & -,2 \\  \hline
        PNN &  0.8038 & 0.4927 & -,2 &  0.8445 & 0.3424 & -,3&  0.8321 & 0.2775 & -,3\\  \hline
        DeepFM &  0.8025 & 0.4468 & -,2&  0.8481 & 0.3333 & -,2 &  0.8376 & 0.2671 & -,3  \\  \hline
        xDeepFM & \textbf{0.8052}  &  \textbf{0.4418} &  3,2& \textbf{0.8639}  &  \textbf{0.3156} &  3,3 & \textbf{0.8400}  &  \textbf{0.2649} &  3,2  \\ \hline   \hline   
     \end{tabular}
     }}
     \label{tab:overallperformance} 
\end{table*}
\subsection{Performance Comparison among Individual Neural Components (Q1)}
We want to know how CIN performs individually. Note that FM measures 2-order feature interactions explicitly, DNN model high-order feature interactions implicitly, CrossNet tries to model high-order feature interactions with a small number of parameters (which is proven not effective in Section \ref{sec:explicit}), and CIN models high-order feature interactions explicitly. There is no theoretic guarantee of the superiority of one individual model over the others, due to that it really depends on the dataset. For example, if the practical dataset does not require high-order feature interactions, FM may be the best individual model. Thus we do not have any expectation for which model will perform the best in this experiment. 
 
 \indent Table \ref{tab:cin} shows the results of individual models on the three practical datasets. Surprisingly, our CIN outperform the other models consistently. On one hand, the results indicate that for practical datasets, higher-order interactions over sparse features are necessary, and this can be verified through the fact that DNN, CrossNet and CIN outperform FM significantly on all the three datasets.  On the other hand, CIN is the best individual model, which demonstrates the effectiveness of CIN on modeling explicit high-order feature interactions. Note that a $k$-layer CIN can model $k$-degree feature interactions. It is also interesting to see that it take 5 layers for CIN to yield the best result ON the Bing News dataset.
\subsection{Performance of Integrated Models (Q2)}
xDeepFM integrates CIN and DNN into an end-to-end model. While CIN and DNN covers two distinct properties in learning feature interactions, we are interested to know whether it is indeed necessary and effective to combine them together for jointly explicit and implicit learning. Here we compare several strong baselines which are not limited to individual models, and the results are shown in Table \ref{tab:overallperformance}. We observe that LR is far worse than all the rest models, which demonstrates that factorization-based models are essential for measuring sparse features. Wide\&Deep, DCN, DeepFM and xDeepFM are significantly better than DNN, which directly reflects that,  despite their simplicity, incorporating hybrid components are important for boosting the accuracy of predictive systems. Our proposed xDeepFM achieves the best performance on all datasets, which demonstrates that combining explicit and implicit high-order feature interaction is necessary, and xDeepFM is effective in learning this class of combination. Another interesting observation is that, all the neural-based models do not require a very deep network structure for the best performance. Typical settings for the depth hyper-parameter are 2 and 3, and the best depth setting for xDeepFM is 3, which indicates that the interactions we learned are at most 4-order.
 \subsection{Hyper-Parameter Study (Q3)}
 We study the impact of hyper-parameters on xDeepFM in this section, including (1) the number of hidden layers; (2) the number of neurons per layer; and (3) activation functions. We conduct experiments via holding the best settings for the DNN part while varying the settings for the CIN part.
 \begin{figure*}[htbp]
\centering
\begin{subfigure}{.32\textwidth}
  \centering
  \includegraphics[width=0.85\textwidth,height=.5\textwidth]{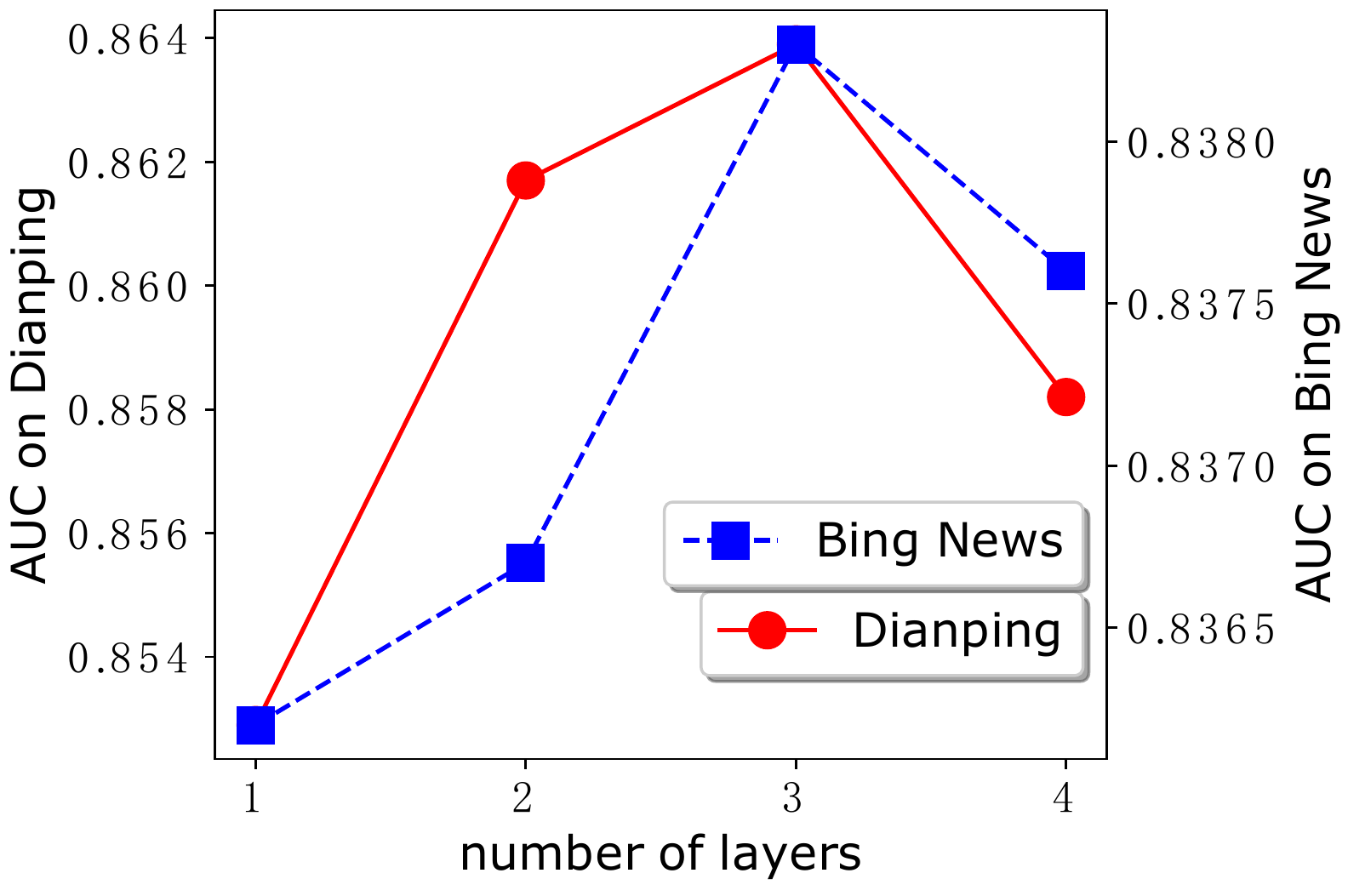}
  \caption{Number of layers.}
  \label{fig:auc_depth}
\end{subfigure} \hfill  
\begin{subfigure}{.32\textwidth}
  \centering
  \includegraphics[width=0.85\textwidth,height=.5\textwidth]{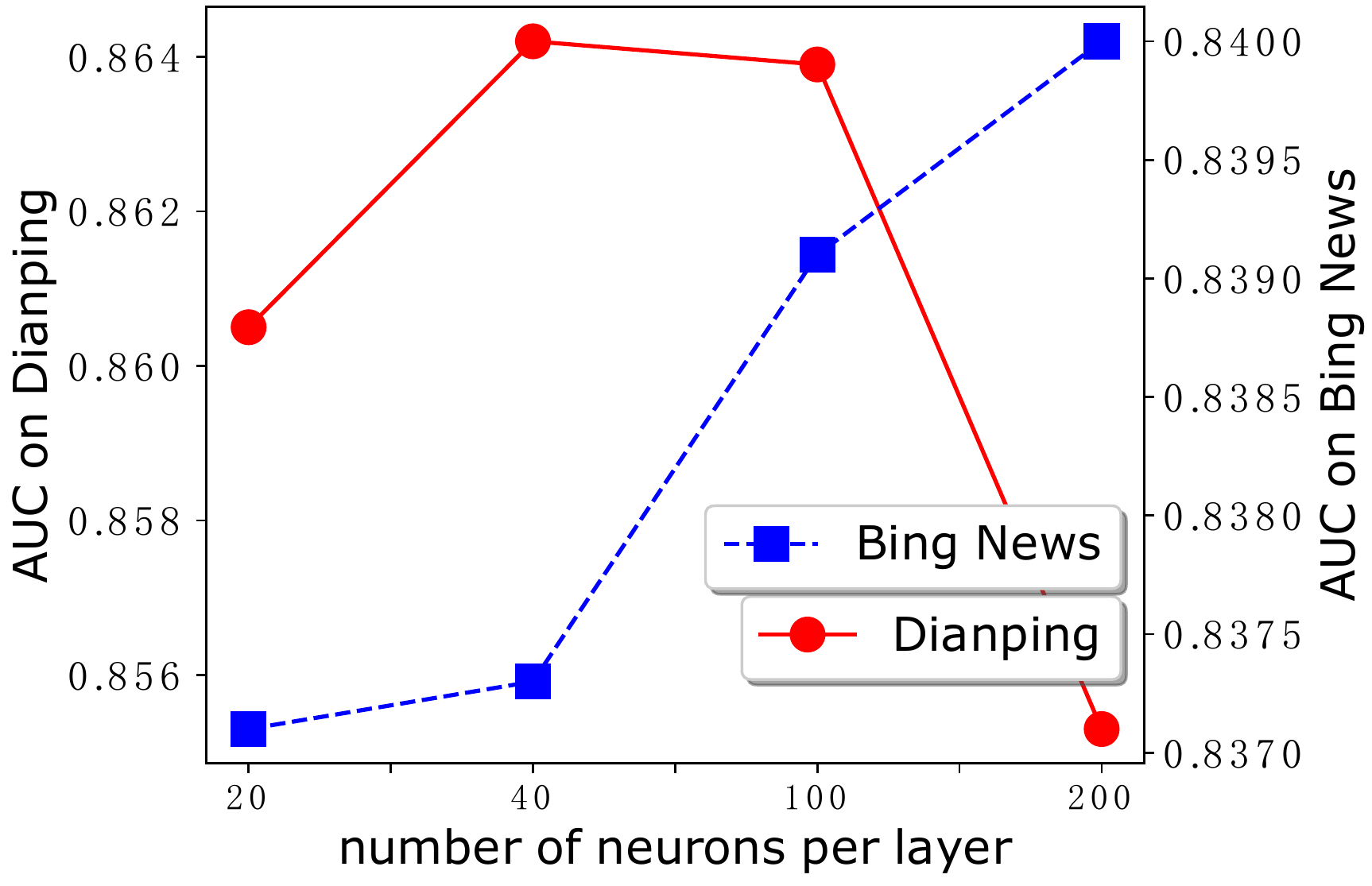}
  \caption{Number of neurons per layer.}
  \label{fig:auc_size}
\end{subfigure}    \hfill  
\begin{subfigure}{.32\textwidth}
  \centering
  \includegraphics[width=0.85\textwidth,height=.5\textwidth]{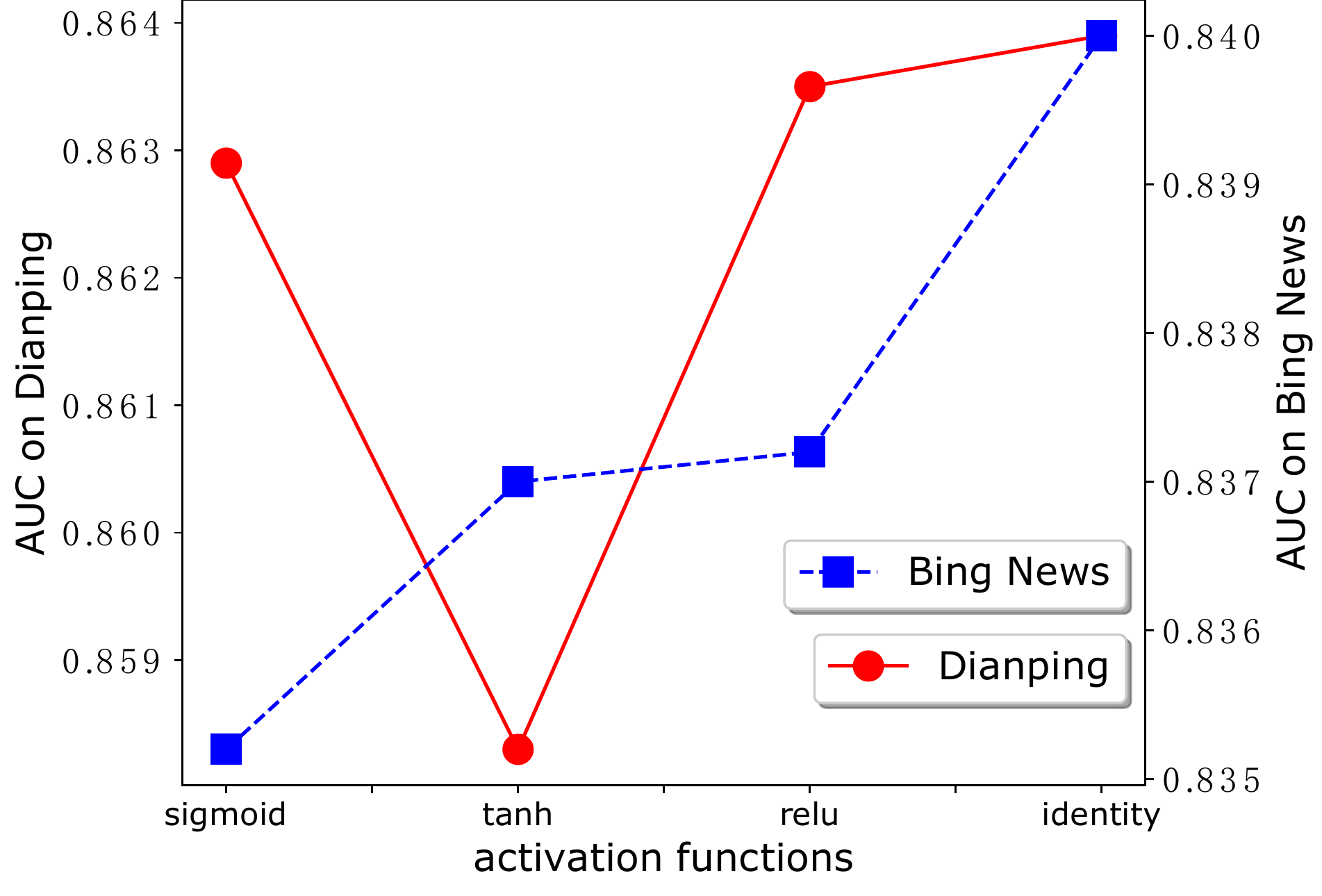}
  \caption{Activation functions}
  \label{fig:auc_activation}
\end{subfigure} 
\caption{Impact of network hyper-parameters on AUC performance.}
\label{fig:auc_hyperparameter}
\end{figure*}
 \begin{figure*}[htbp]
\centering
\begin{subfigure}{.32\textwidth}
  \centering
  \includegraphics[width=0.85\textwidth,height=.5\textwidth]{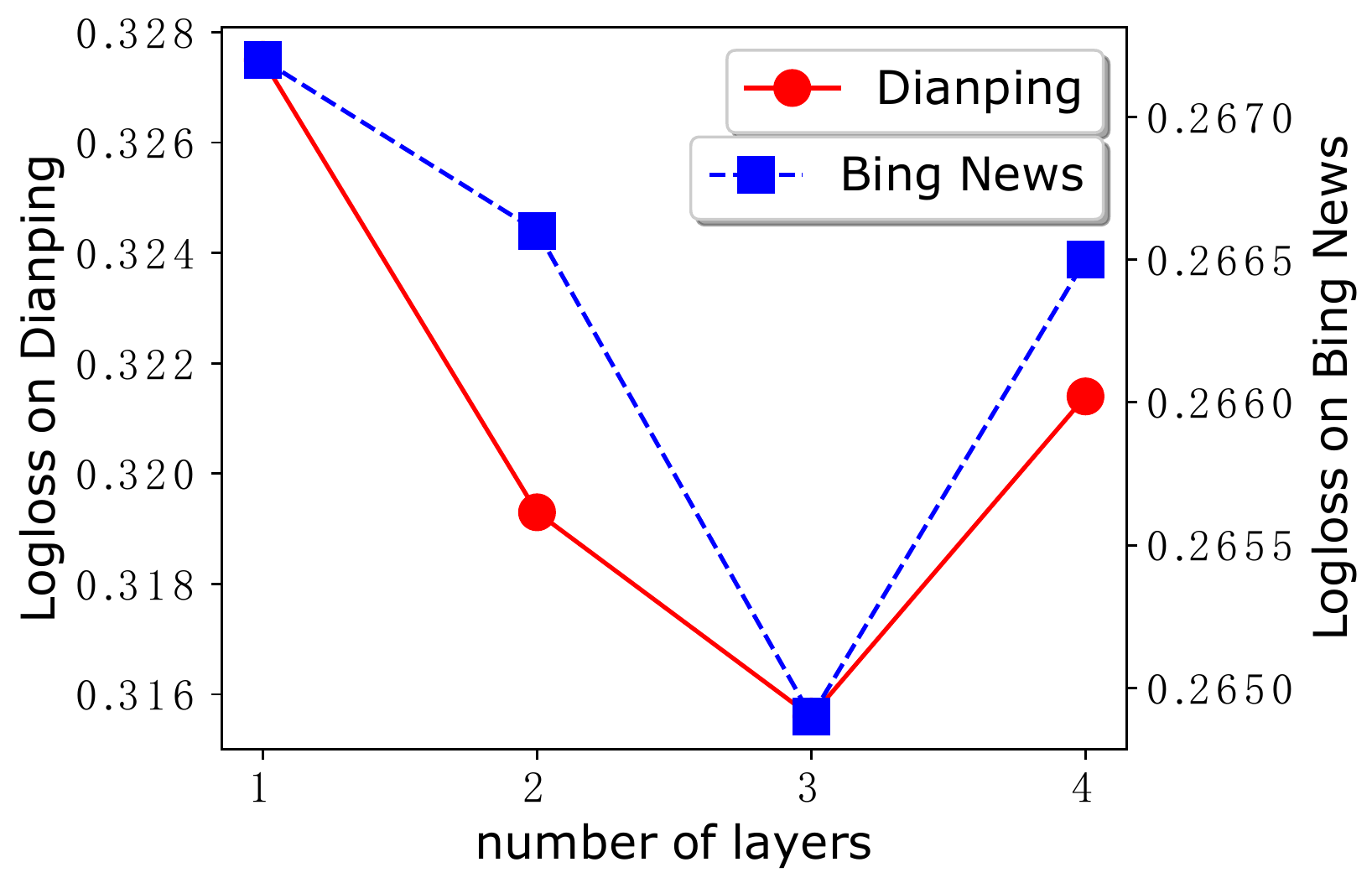}
  \caption{Number of layers.}
  \label{fig:logloss_depth}
\end{subfigure} \hfill  
\begin{subfigure}{.32\textwidth}
  \centering
  \includegraphics[width=0.85\textwidth,height=.5\textwidth]{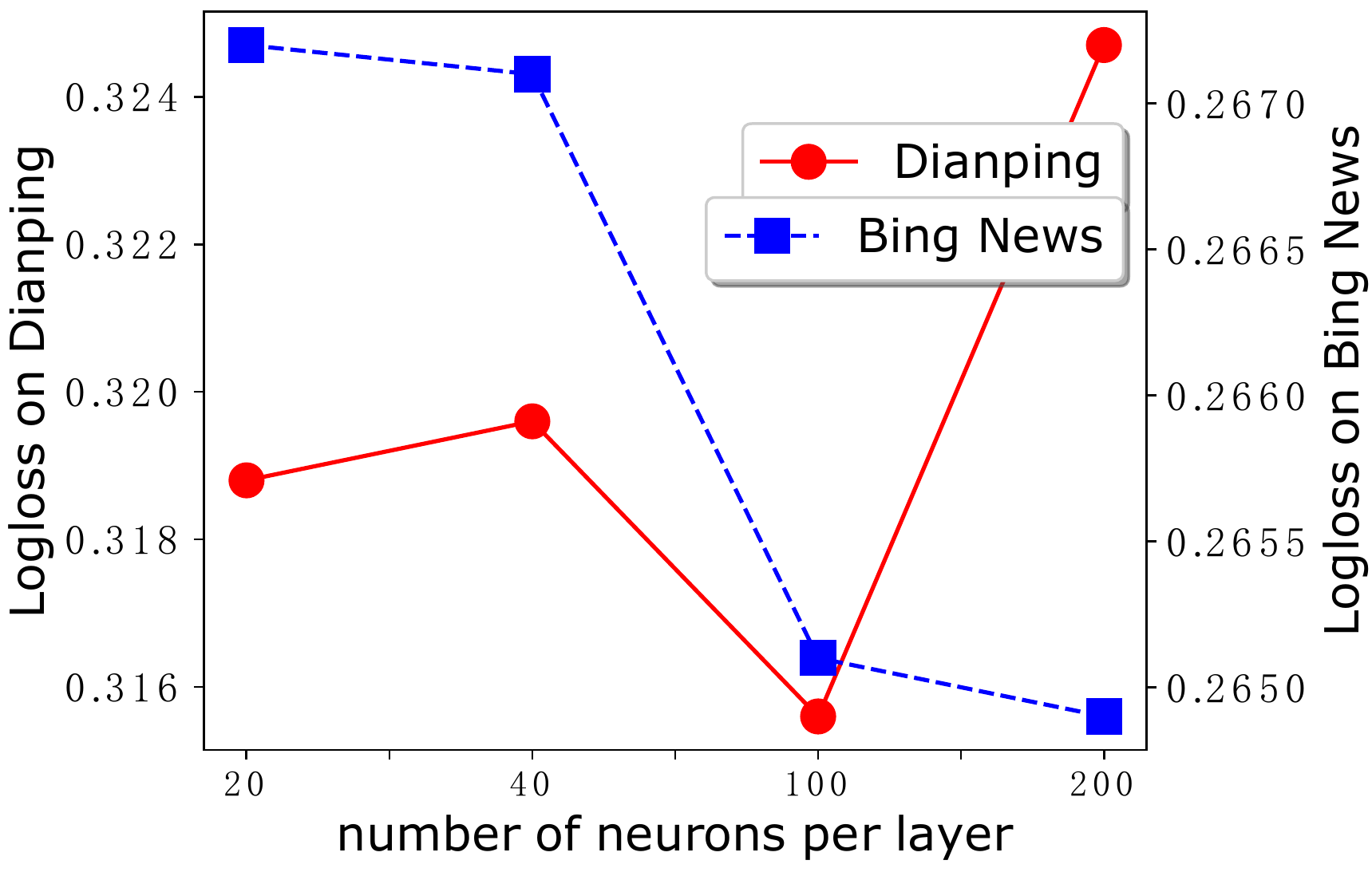}
  \caption{Number of neurons per layer.}
  \label{fig:logloss_size}
\end{subfigure}    \hfill  
\begin{subfigure}{.32\textwidth}
  \centering
  \includegraphics[width=0.85\textwidth,height=.5\textwidth]{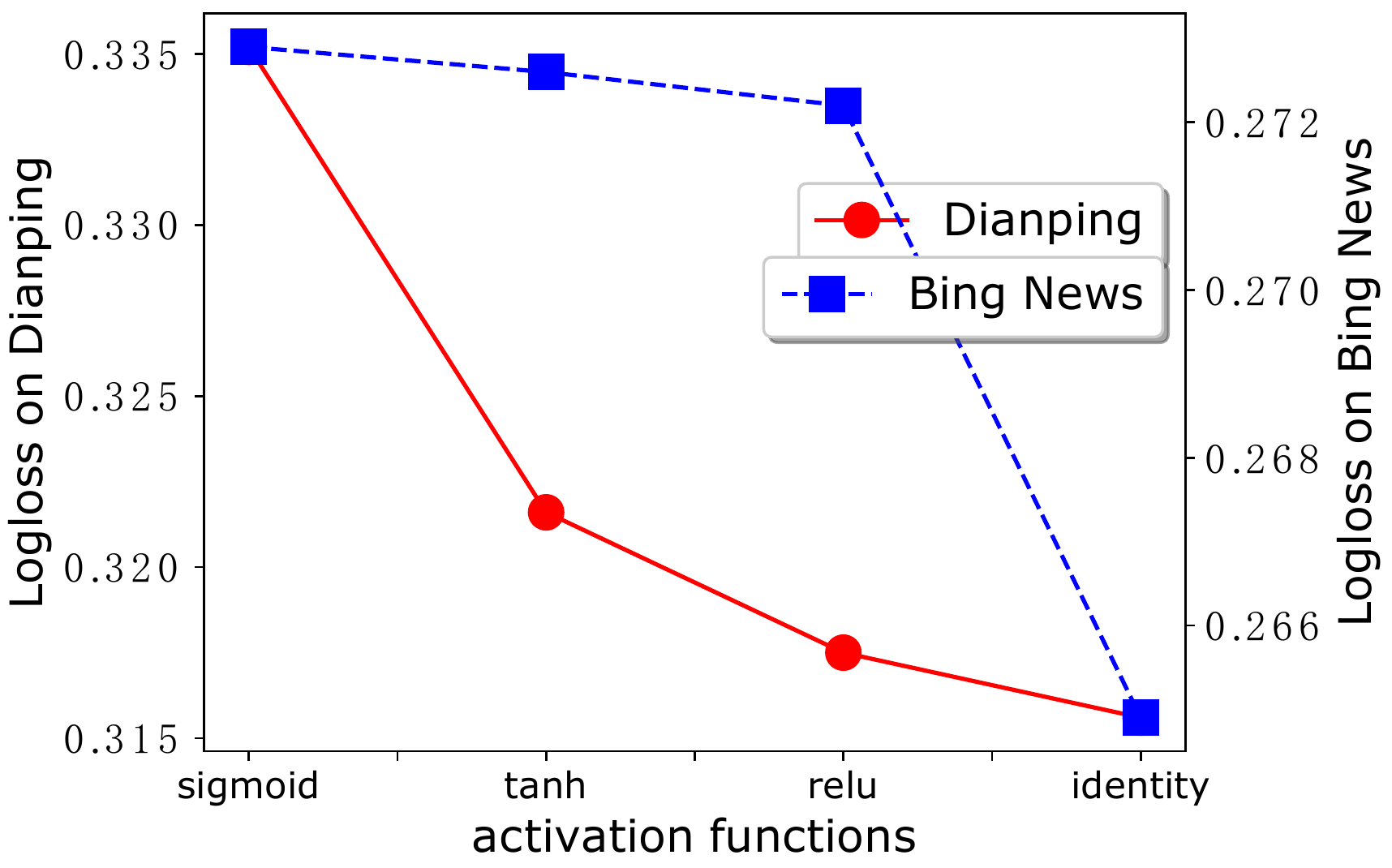}
  \caption{Activation functions}
  \label{fig:logloss_activation}
\end{subfigure} 
\caption{Impact of network hyper-parameters on Logloss performance.}
\label{fig:logloss_hyperparameter}
\end{figure*}  
 \\\indent\textbf{Depth of Network}. Figure \ref{fig:auc_depth} and \ref{fig:logloss_depth} demonstrate the impact of number of hidden layers. We can observe that the performance of xDeepFM increases with the depth of network at the beginning. However, model performance degrades when the depth of network is set greater than 3. It is caused by overfitting evidenced by that we notice that the loss of training data still keeps decreasing when we add more hidden layers.
 \\\indent\textbf{Number of Neurons per Layer}. Adding the number of neurons per layer indicates increasing the number of feature maps in CIN. As shown in Figure \ref{fig:auc_size} and \ref{fig:logloss_size}, model performance on Bing News dataset increases steadily when we increase the number of neurons from $20$ to $200$, while on Dianping dataset, $100$ is a more suitable setting for the number of neurons per layer. In this experiment we fix the depth of network at 3.
 \\\indent\textbf{Activation Function}. Note that we exploit the identity as activation function on neurons of CIN, as shown in Eq. \ref{eq:xk}. A common practice in deep learning literature is to employ non-linear activation functions on hidden neurons. We thus compare the results of different activation functions on CIN (for neurons in DNN, we keep the activation function with \textsl{relu}). As shown in Figure \ref{fig:auc_activation} and \ref{fig:logloss_activation}, identify function is indeed the most suitable one for neurons in CIN.
\section{related work}\label{relatedwork}
\subsection{Classical Recommender Systems}
\subsubsection{Non-factorization Models}
For web-scale recommender systems (RSs), the input features are usually sparse, categorical-continuous-mixed, and high-dimensional. Linear models, such as logistic regression with \textsl{FTRL} \cite{mcmahan2013ad}, are widely adopted as they are easy to manage, maintain, and deploy. Because linear models lack the ability of learning feature interactions, data scientists have to spend a lot of work on engineering cross features in order to achieve better performance \cite{richardson2007predicting,Lian:2017:PLJ:3124791.3124794}. Considering that some hidden features are hard to design manually, some researchers exploit boosting decision trees to help build feature transformations \cite{he2014practical,ling2017model}.
\subsubsection{Factorization Models}
A major downside of the aforementioned models is that they can not generalize to unseen feature interactions in the training set. Factorization Machines \cite{rendle2010factorization} overcome this problem via embedding each feature into a low dimension latent vector. Matrix factorization (MF) \cite{koren2009matrix}, which only considers IDs as features, can be regarded as a special kind of FM. Recommendations are made via the product of two latent vectors, thus it does not require the co-occurrence of user and item in the training set.  MF is the most popular model-based collaborative filtering method in the RS literature \cite{srebro2005maximum,koren2008factorization,lee2013local,pan2008one}. \cite{chen2012svdfeature,menon2010log} extend MF to leveraging side information, in which both a linear model and a MF model are included. On the other hand, for many recommender systems, only implicit feedback datasets such as users' watching history and browsing activities are available. Thus researchers extend the factorization models to a Bayesian Personalized Ranking (BPR) framework \cite{rendle2009bpr,rendle2010pairwise,he2016vbpr,yuan2016lambdafm} for implicit feedback.
\subsection{Recommender Systems with Deep Learning}
Deep learning techniques have achieved great success in computer vision \cite{krizhevsky2012imagenet,he2016deep}, speech recognition \cite{hinton2012deep,amodei2016deep} and natural language understanding \cite{mikolov2010recurrent,cho2014learning}. As a result, an increasing number of researchers are interested in employing DNNs for recommender systems.
\subsubsection{Deep Learning for High-Order Interactions} To avoid manually building up high-order cross features, researchers apply DNNs on field embedding, thus patterns from categorical feature interactions can be learned automatically. Representative models include FNN \cite{zhang2016deep}, PNN \cite{qu2016product}, DeepCross \cite{shan2016deep}, NFM \cite{he2017neural}, DCN \cite{wang2017deep}, Wide\&Deep \cite{cheng2016wide}, and DeepFM \cite{guo2017deepfm}.  These models are highly related to our proposed xDeepFM. Since we have reviewed them in Section \ref{introduction} and Section \ref{preliminaries}, we do not further discuss them in detail in this section. We have demonstrated that our proposed xDeepFM has two special properties in comparison with these models: (1) xDeepFM learns high-order feature interactions in both explicit and implicit fashions; (2) xDeepFM learns feature interactions at the vector-wise level rather than at the bit-wise level.
\subsubsection{Deep Learning for Elaborate Representation Learning} We include some other deep learning-based RSs in this section due to that they are less focused on learning feature interactions. Some early work employs deep learning mainly to model auxiliary information, such as visual data \cite{he2016vbpr} and audio data \cite{wang2014improving}. Recently, deep neural networks are used to model the collaborative filtering (CF) in RSs. \cite{he2017neuralwww} proposes a Neural Collaborative Filtering (NCF) so that the inner product in MF can be replaced with an arbitrary function via a neural architecture. \cite{sedhain2015autorec,wu2016collaborative} model CF base on the autoencoder paradigm, and they have empirically demonstrated that autoencoder-based CF outperforms several classical MF models. Autoencoders can be further employed for jointly modeling CF and side information with the purpose of generating better latent factors \cite{dong2017hybrid,wang2015collaborative,zhang2016collaborative}. \cite{elkahky2015multi,lian2017cccfnet} employ neural networks to jointly train multiple domains' latent factors. \cite{chen2017attentive} proposes the Attentive Collaborative Filtering (ACF) to learn more elaborate preference at both item-level and component-level. \cite{zhou2017deep} shows tha traditional RSs can not capture \textsl{interest diversity} and \textsl{local activation} effectively, so they introduce a Deep Interest Network (DIN) to represent users' diverse interests with an attentive activation mechanism.
\section{Conclusions}\label{conclusions}
In this paper, we propose a novel network named Compressed Interaction Network (CIN), which aims to learn high-order feature interactions explicitly. CIN has two special virtues: (1) it can learn certain bounded-degree feature interactions effectively; (2) it learns feature interactions at a vector-wise level. Following the spirit of several popular models, we incorporate a CIN and a DNN in an end-to-end framework, and named the resulting model eXtreme Deep Factorization Machine (xDeepFM). Thus xDeepFM can automatically learn high-order feature interactions in both explicit and implicit fashions, which is of great significance to reducing manual feature engineering work. We conduct comprehensive experiments and the results demonstrate that our xDeepFM outperforms state-of-the-art models consistently on three real-world datasets.\\
\indent There are two directions for future work. First, currently we simply employ a sum pooling for embedding multivalent fields. We can explore the usage of the DIN mechanism \cite{zhou2017deep} to capture the related activation according to the candidate item. Second, as discussed in section \ref{sec:time_complexity}, the time complexity of the CIN module is high. We are interested in developing a distributed version of xDeepFM which can be trained efficiently on a GPU cluster.
\section*{Acknowledgements}
The authors would like to thank the anonymous reviewers for their insightful reviews, which are very helpful on the revision of this paper. This work is supported in part by Youth Innovation Promotion Association of CAS. 

\bibliographystyle{ACM-Reference-Format}
\bibliography{reference}
\balance
\end{document}